\documentclass[conference]{IEEEtran}
\IEEEoverridecommandlockouts

\usepackage{graphicx}
\usepackage[misc,geometry]{ifsym} 

\usepackage[utf8]{inputenc} 
\usepackage[T1]{fontenc}    

\usepackage{url}            
\usepackage{booktabs}       
\usepackage{amsfonts}       
\usepackage{nicefrac}       
\usepackage{microtype}      

\usepackage{times}
\usepackage{soul}
\usepackage{url}
\usepackage{caption}

\usepackage{amsthm,amsmath}
\usepackage{booktabs}
\usepackage{algorithm}
\usepackage[noend]{algpseudocode}
\usepackage{tikz}
\usetikzlibrary{positioning}
\usepackage{pgfplots}
\usepackage{microtype}
\usepackage[noend]{algpseudocode}
\usepackage{subcaption}
\graphicspath{ {Results/} }

\usepackage{enumitem, hyperref}
\makeatletter
\def\namedlabel#1#2{\begingroup
	#2%
	\def\@currentlabel{#2}%
	\phantomsection\label{#1}\endgroup
}
\makeatother

\usepackage{nccmath}
\usepackage{mathtools}
\usepackage{bbm}
\pgfplotsset{compat=1.16}

\def\beq{\begin{equation}}\def\eeq{\end{equation}}\usepackage[T1]{fontenc}

\DeclareMathOperator*{\argmax}{arg\,max}

\newcommand{\cX}{{\cal X}}
\newcommand{\cS}{{\cal S}}

\newcommand{\cT}{{\mathcal T}}

\newtheorem{theorem}{Theorem}
\newtheorem{definition}{Definition}[section]
\newtheorem{lemma}[theorem]{Lemma}

\usepackage[switch]{lineno}

\usepackage{todonotes}
\renewcommand{\vec}[1]{\mathbf{#1}}
\newcommand*{\TwoRows}[2]{{\begin{tabular}[c]{@{}l@{}} {#1} \\ {#2} \end{tabular}}}{}

\begin{document}

\title{Efficient Data Analytics on Augmented Similarity Triplets}

\author{\IEEEauthorblockN{Sarwan Ali$^{*1}$, Muhammad Ahmad$^{*2}$, Umair ul Hassan$^{3}$ , Muhammad Asad Khan$^{4}$, Shafiq Alam$^{5}$, Imdadullah Khan$^{2+}$}
\IEEEauthorblockA{
Georgia State University, Atlanta, USA$^{1}$ \\
Lahore University of Management Science, Lahore, Pakistan$^{2}$\\
2B 6VYCCM KNUniversity of Galway, Galway, Ireland $^{3}$ \\
Hazara University, Mansehra, Pakistan$^{4}$\\
Massey University, Auckland, New Zealand$^{5}$ \\
sali85@student.gsu.edu, 17030056@lums.edu.pk, 
umair.ulhassan@universityofgalway.ie \\
asadkhan@hu.edu.pk, SAlam1@massey.ac.nz, imdad.khan@lums.edu.pk} 
\\
{$^{*}$ Equal Contribution}
{$^{+}$ Corresponding Author}
\thanks{2022 IEEE International Conference on Big Data (Big Data) | 978-1-6654-8045-1/22/\$31.00 \copyright 2022 European Union}
}

\maketitle

\thispagestyle{plain}
\pagestyle{plain}
\begin{abstract}
Data analysis requires a pairwise proximity measure over objects. Recent work has extended this to situations where the distance information between objects is given as comparison results of distances between three objects (triplets). Humans find comparison tasks much easier than the exact distance computation, and such data can be easily obtained in big quantities via crowdsourcing. In this work, we propose triplets augmentation, an efficient method to extend the triplets data by inferring the hidden implicit information from the existing data. Triplets augmentation improves the quality of kernel-based and kernel-free data analytics. We also propose a novel set of algorithms for common data analysis tasks based on triplets. These methods work directly with triplets and avoid kernel evaluations, thus are scalable to big data. We demonstrate that our methods outperform the current best-known techniques and are robust to noisy data.
\end{abstract}


\section{Introduction}\label{section:introduction}

To extract knowledge from data, it is generally assumed that input data is drawn from a vector space with a well-defined pairwise proximity measure~\cite{ali2021predicting}. However, big data comes in a wide variety; and, in many cases (e.g., sequences, images, and text), the data objects are not given as feature vectors. Such data needs to be mapped to a meaningful feature space, where vector-space-based data mining algorithms can be employed. Many data mining and machine learning methods, such as support vector machines (SVM), kernel-PCA, and agglomerative clustering, do not explicitly require input data as feature vectors; instead, they only utilize the pairwise distance information~\cite{ali2022efficient}. Moreover, the choice of a distance or similarity measure is often arbitrary and does not necessarily capture the inherent similarity between objects.

Recently, a relaxed paradigm of performing data analytics from ordinal pairwise distance data has emerged~\cite{Kleindessner2017Kernel}.
The primary motivation for using distance comparison data is to rely on human judgments about the qualitative similarity between objects in human-based computation settings (e.g., crowdsourcing, a prominent source of Big data). Unlike computers, it is widely accepted that humans are better at making qualitative assessments, such as determining two images to be ``perceptually'' similar. Human annotations capture proximity information latent in data objects. Since it is inherently easier to compare two distances than to actually compute the exact distance~\cite{stewart2005absolute}, comparison-based distance information tends to be more accurate and easier to obtain. For instance, in Figure~\ref{fig:3cars3triplets}, one can infer that $\vec{c}_1$ is more similar to $\vec{c}_2$, than it is to $\vec{c}_3$ (based on vehicle's utility), even though, a pixel-based distance measure might bring about very different results.

\begin{figure}[!t]
	\centering
	\includegraphics[width=.96\linewidth,page=10]{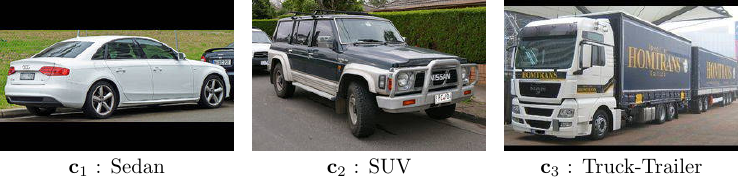}
	\caption{$\vec{c}_1$ is more similar to $\vec{c}_2$ than to $\vec{c}_3$ (by vehicle use).
}
	\label{fig:3cars3triplets}
\end{figure}

A similarity triplet $(\vec{x}, \vec{y}, \vec{z})$ encodes ordering among the three pairwise distances between three objects $\vec{x},\vec{y},$ and $\vec{z}$. Given a set ${\mathcal X}$ of abstractly described $n$ objects, the pairwise distance information is provided by a collection of similarity triplets, ${\mathcal T}$. To analyze such data, the traditional approach first represents ${\mathcal X}$ as points in Euclidean space (ordinal embedding) that {\em preserves} the distance information in the triplets~\cite{tamuz2011adaptively,AmidVW16,NIPS2016_6554}. Next, appropriate vector-space-based algorithms are employed to solve the analytical problem at hand.

The kernel-based machine learning approach maps the objects in $\cX$ to feature vectors based on $\cT$. A kernel function between two objects, which is usually a dot-product between the corresponding feature vectors, serves as the pairwise similarity value~\cite{Kleindessner2017Kernel}. This approach differs from ordinal embedding in that it does not try to satisfy some global fit to the data but instead simply represents the object itself. These feature vectors are high-dimensional which makes kernel evaluation computationally infeasible. Computing the kernel with the best-known matrix multiplication algorithm takes $\Omega(n^{3.376})$ time, which clearly does not scale to a number of objects more than a few thousand~\cite {Kleindessner2017Kernel}. 

In this paper, we propose methods for data analysis directly on triplets, thus avoiding kernel computation. The proposed method has linear complexity in the number of triplets and is highly scalable and parallelizable. As a result, we overcome the significant challenges of computational complexity and scalability of kernel methods in big data.  Empirical evaluation shows that our kernel-free clustering and classification algorithms outperform competitor methods. We also propose a data augmentation scheme that infers more sound triplets derived from input data without the need for expensive data gathering from human judges. Our data augmentation method can also help identify data inconsistencies in the input and it can lead to better data collection strategies; futhermore, it is of independent interest in crowdsourcing settings. Since the data provided by humans are prone to data veracity issues, we use the most realistic \emph{relative noise model} to introduce noisy triplets in the input and evaluate if our methods are robust to error.

The main contributions of this work are as follows:

\begin{itemize}
	
	\item We compute the \textit{closeness} of objects from fixed objects, directly from triplets. This closeness is used to find objects' nearest neighbors for classification and clustering. We show that classification and clustering algorithms based on closeness perform better than the baseline methods.

	\item We propose a robust and efficient method for augmentation of triplets to add more sound triplets to input data. The benefits of augmentation are two-fold: (i) reduction in the cost of collecting triplets, and (ii) improved quality of kernel-based and kernel-independent data analytics. 
	\item We provide linear-time algorithms both for data augmentation and data analysis from triplets.
	\item We use \emph{relative noise model} to introduce erroneous triplets in the input. The experimental results show that our methods are robust to error, and our algorithms retain their quality to a greater extent with increasing error in input.
\end{itemize}

The rest of the paper is organized as follows. 
The related work is discussed in Section~\ref{section:related_work}. 
We formulate the problem in Section~\ref{section:problem_formulation} and discuss our feature vector representation method in Section~\ref{section:algorithm}. 
We further discuss the proposed triplet-based method for data analysis in Section~\ref{sec_triplets_analysis}.
We provide details of  dataset statistics and baseline methods in  Section~\ref{section:experimental_evaluation}. 
We report experimental results and comparisons with the baselines in Section~\ref{sec_results_discussion}. 
Finally, we conclude the paper in Section~\ref{section:conclusion}.

\section{Related work}\label{section:related_work}

Evaluating proximity between pairs of objects is a fundamental building block in machine learning and data analysis. Since comparing two distances is fundamentally easier than computing actual distances~\cite{stewart2005absolute}, recent works utilize {\em `relative similarity'} among objects rather than working on actual pairwise similarities.

Relative similarities among objects are used to generate ordinal embedding -- a representation that preserves the orderings of distances among pairs of objects~\cite{AmidU15,HeimBSH15,tamuz2011adaptively}. This representation, usually in a low-dimensional Euclidean space, is used to solve the specific machine learning problem at hand. For instance, low-rank embedding based on convex optimization is detailed in~\cite{AgarwalWCLKB07}. Ordinal embedding with the objective based on the number of constraint violations is discussed in~\cite{TeradaL14}. In this approach, the quantity and quality of triplets are of pivotal importance to extract meaningful information~\cite{arias2017some}. A lower bound of $\omega(n\log n)$ is derived in~\cite{jamieson2011low} to get useful knowledge from triplets. Methods to learn embedding with bounded errors from noisy similarity information are proposed in~\cite{AmidVW16,NIPS2016_6554,Kleindessner2017lens}, while~\cite{wilber2014cost} presents techniques to get higher quality triplets via crowd-sourcing.

As an alternative to ordinal embeddings, kernel functions based on similarity triplets are proposed in~\cite{Kleindessner2017Kernel}. The bottleneck, however, is the kernel computation through matrix multiplication which takes time proportional to $O(n^{3.376})$ in the dense cases. An approximation of multidimensional scaling, called Landmark technique~\cite{de2004sparse}, reduces the search space to a chosen useful subset of data points. The time complexity of kernel computation is still asymptotically super-cubic, making the kernel-based approach infeasible for big data.

Some recent works perform data analytics directly on triplets without computing the kernels. An approximate median of the dataset using triplets is proposed in~\cite{heikinheimo2013crowd}. Algorithms to estimate density using relative similarity are provided in~\cite{ukkonen2015crowdsourced}.  Similarly, triplets are used to find approximate median and outliers in a dataset~\cite{Kleindessner2017lens}. Approximate nearest neighbors-based clustering and classification algorithms are also proposed in~\cite{Kleindessner2017lens}. Another set of research works perform machine learning directly based on comparison trees~\cite{haghiri2017ComparisonBasedNearestNeighbor,HaghiriGL18,PerrotL19}. Active triplets generation mechanism is used in~\cite{tamuz2011adaptively,haghiri2017ComparisonBasedNearestNeighbor,HaghiriGL18} in which the triplets of the desired choice are queried. Machine learning tasks performed directly on triplets include nearest neighbors search~\cite{haghiri2017ComparisonBasedNearestNeighbor}, classification ~\cite{PerrotL19,HaghiriGL18}, comparison based hierarchical clustering~\cite{ghoshdastidar2018foundations}, and correlation clustering~\cite{ukkonen2017crowdsourced}.

Data gathering in this setting is expensive as it requires human judges to order objects in a triplet~\cite{GomesWKP11,ParameswaranGPPRW12}. This leads to reduced quality in inferences drawn from limited information. Dense feature vectors and kernels based on them enhance the quality of machine learning algorithms. Furthermore, dense feature vectors can lead to approximation algorithms for computing kernels with quality guarantees. For literature on large scale kernel learning~\cite{BojarskiCCFGMSS17,sun2018but} and kernel approximation~\cite{FarhanNIPS2017,LiTOS19} see references therein. Authors in~\cite{li2016extracting} demonstrate that more noise and inconsistencies in human responses regarding the pairwise rating of data could corrupt similarity judgments, both within- and across subjects, which could badly affect the overall data interpretation.


\section{Problem Formulation}\label{section:problem_formulation}
Let ${\mathcal X} = \{x_1,\ldots, x_n\}$ be a set of $n$ objects in an arbitrary but fixed order, and let $d:{\mathcal X}\times {\mathcal X} \mapsto \mathbb{R}^+\cup \{0\}$ be a distance measure on $\cX$. We do not require $d$ to be a metric and only assume that $\forall \, \vec{x},\vec{y}\in {\mathcal X}\,:\,$
\begin{equation*}
d(\vec{x},\vec{y})\geq 0, \;\;
d(\vec{x},\vec{y}) = 0 \leftrightarrow \vec{x} = \vec{y},\text{ and } \;
d(\vec{x},\vec{y}) = d(\vec{y},\vec{x})
\end{equation*}
The distance measure $d$ or the corresponding similarity measure $sim$ (an inverse of $d$) is not provided. A similarity triplet $(\vec{x}, \vec{y}, \vec{z})$ encodes ordinal information about distances between three objects $\vec{x},\vec{y},$ and $\vec{z}$. The ordering information in a triplet can be of three different forms~\cite{Kleindessner2017lens}, as described in Table~\ref{tbl:tripletDefinitions}.
\begin{table}[h!]
\centering
\begin{tabular}{cll}
\toprule
Type & Notation & Description \& Definition \\ \midrule
\namedlabel{tripAnchor}{{$\pmb{\mathbb{A}}$}}: &
$(\vec{x},\vec{y},\vec{z})_A$ & \TwoRows{$\vec{x} \text{ is closer to } \vec{y} \text{ than to } \vec{z}$}{$d(\vec{x}, \vec{y}) < d(\vec{x}, \vec{z}) $} \\
\namedlabel{tripCentral}{{$\pmb{\mathbb{C}}$}}: &
$(\vec{x},\vec{y},\vec{z})_C$ & \TwoRows{$\vec{x}$ is the central element}{$d(\vec{x},\vec{y}) < d(\vec{y},\vec{z})$  and $d(\vec{x},\vec{z}) < d(\vec{y},\vec{z})$} \\
\namedlabel{tripOutlier}{{$\pmb{\mathbb{O}}$}}: &
$(\vec{x},\vec{y},\vec{z})_O$ & \TwoRows{$\vec{x}$ is the outlier}{ $ d(\vec{x},\vec{y}) > d(\vec{y},\vec{z})$ and $d(\vec{x},\vec{z}) > d(\vec{y},\vec{z})$}\\
\bottomrule
\end{tabular}
	\caption{Definitions of three types of triplets encoding different ordinal information about pairwise distances.} \label{tbl:tripletDefinitions}
\end{table}

The only information about $d$ or $sim$ is provided as a collection $\mathcal{T}$ of triplets of the type~\ref{tripAnchor},~\ref{tripCentral}, or~\ref{tripOutlier}. An illustration of three forms of triplets is shown in Figure~\ref{center_and_outlier_triplet_graphically}. From the three images of vehicles in Figure~\ref{fig:3cars3triplets}, the following three triplets can be inferred, $(\vec{c}_1,\vec{c}_2,\vec{c}_3)_A$, $(\vec{c}_2,\vec{c}_1,\vec{c}_3)_C$, and $(\vec{c}_3,\vec{c}_1,\vec{c}_2)_O$.
Observe that a triplet of type~\ref{tripCentral} or~\ref{tripOutlier} provides relative orderings of three pairs of objects and is stronger than that of the type~\ref{tripAnchor}, which provides relative orderings of two pairs. More formally,
\begin{align}
\label{tripCent2Anchor} \nonumber (\vec{x},\vec{y},\vec{z})_C &\iff d(\vec{x},\vec{y}) < d(\vec{y},\vec{z}) \wedge d(\vec{x},\vec{z}) < d(\vec{y},\vec{z}) \\&\iff (\vec{y},\vec{x},\vec{z})_A \wedge (\vec{z},\vec{x},\vec{y})_A
\end{align}
Similarly, a type~\ref{tripOutlier} triplet is equivalent to two type~\ref{tripAnchor} triplets.
\begin{align}
\label{tripOut2Anchor} \nonumber (\vec{x},\vec{y},\vec{z})_O &\iff d(\vec{x},\vec{y}) > d(\vec{y},\vec{z}) \wedge d(\vec{x},\vec{z}) > d(\vec{y},\vec{z}) \\& \iff (\vec{y},\vec{z},\vec{x})_A \wedge (\vec{z},\vec{y},\vec{x})_A
\end{align}

In this paper, we focus on type~\ref{tripAnchor} triplets, and when input $\cT$ is of the form~\ref{tripCentral} or~\ref{tripOutlier}, we translate it to a collection of triplets of the type~\ref{tripAnchor} using Equation  \eqref{tripCent2Anchor} or \eqref{tripOut2Anchor}. For notational convenience, we still refer to input as $\cT$.

\begin{figure}[!t]
	\centering
	\includegraphics[width=.85\linewidth,page=7]{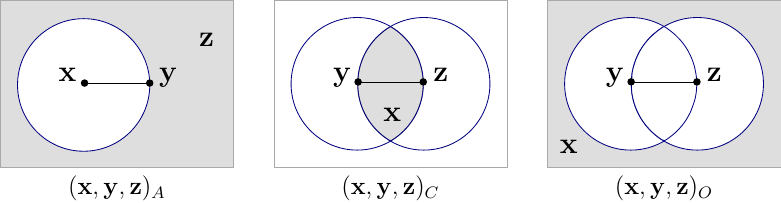}
	\caption{Illustration of relative locations of the three points corresponding to each of the three types of triplets.}
	\label{center_and_outlier_triplet_graphically}
\end{figure}

We refer to the set of all ordered pairs of $\mathcal{X}$ as ${\binom{\mathcal{X}}{2}}$, i.e. $\mathcal{X} = \{(x_i,x_j): x_i,x_j \in \mathcal{X}, i<j\}$. Two alternative mappings of objects in $\mathcal{X}$ to feature vectors in $\{-1,0,1\}^{\binom{n}{2}}$ are given in~\cite{Kleindessner2017Kernel}. The coordinates of the feature vectors correspond to ordered pairs $(\vec{x}_i, \vec{x}_j) \in {\cX \choose 2}$ with $i<j$. For $\vec{x} \in \mathcal X$, the feature value at the coordinate corresponding to $(\vec{x}_i, \vec{x}_j)$ is $1$ if $(\vec{x},\vec{x}_i,\vec{x}_j)\in \mathcal{T}$, $-1$ if $(\vec{x},\vec{x}_j,\vec{x}_i)\in \mathcal{T}$, and $0$ otherwise. More precisely, $\vec{x}\in {\mathcal X}$ is mapped to $\Phi_1(\vec{x})$ as follows:

\begin{equation*}
\Phi_{1}(\vec{x}) =  \bigg( \Phi_{1}(\vec{x})[\gamma]\bigg)_{\text{\footnotesize $\gamma \in  \binom{\mathcal X}{2}$}}, 
\end{equation*}

\begin{equation}
\label{Eq:featuremap}
 \Phi_{1}(\vec{x})[\gamma] = \begin{cases} 1 & \text{if } \gamma = (\vec{x}_i,\vec{x}_j), i< j , (\vec{x},\vec{x}_i,\vec{x}_j) \in {\mathcal T}\\
-1& \text{if } \gamma = (\vec{x}_i,\vec{x}_j), i< j , (\vec{x},\vec{x}_j,\vec{x}_i) \in {\mathcal T}\\
0 & \text{otherwise}
\end{cases}
\end{equation}

In the second feature mapping, $x \in \mathcal{X}$ is mapped to $\Phi_2(x) \in \{-1,0,1\}^{{n\choose 2}}$ as follows:
\begin{equation}
\label{Eq:featuremap2}
\Phi_{2}(\vec{x})[\gamma] = \begin{cases} 1 & \text{ if } \gamma = (\vec{x}_i,\vec{x}_j), i< j , (\vec{x}_i,\vec{x},\vec{x}_j) \in {\mathcal T}\\
-1& \text{ if } \gamma = (\vec{x}_i,\vec{x}_j), i< j , (\vec{x}_i,\vec{x}_j,\vec{x}) \in {\mathcal T}\\
0 & \text{ otherwise }
\end{cases}
\end{equation}

\begin{figure}[!t]
	\centering
	\includegraphics[scale=.55,page=8]{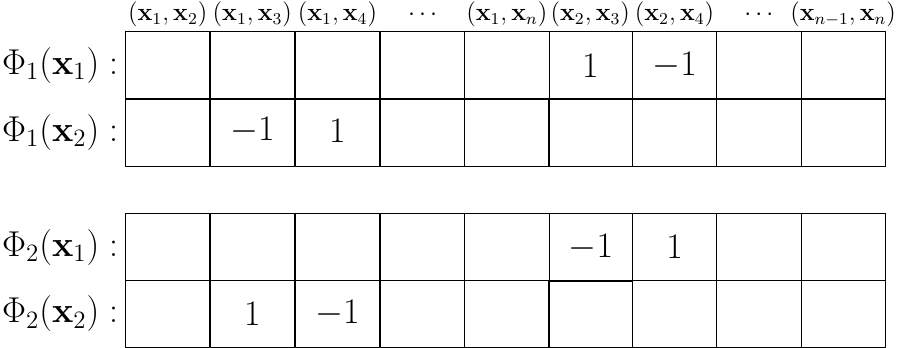}
	
	\caption{$\Phi_1$ and $\Phi_2$ for triplets $\mathcal{T} = \{(\vec{x}_{1}, \vec{x}_{2}, \vec{x}_{3})_A$, $(\vec{x}_{1}, \vec{x}_{4}, \vec{x}_{2})_A$, $(\vec{x}_{2}, \vec{x}_{3}, \vec{x}_{1})_A$, $(\vec{x}_{2}, \vec{x}_{1}, \vec{x}_{4})_A\}$ using \eqref{Eq:featuremap} and \eqref{Eq:featuremap2}.}
	\label{fig:featureVector}
\end{figure}

Figure~\ref{fig:featureVector} shows the feature mapping on a small example. For both feature mappings, the kernel is given as~\cite{Kleindessner2017Kernel}:
\begin{equation}
K_{r}(\vec{x},\vec{y}|{\mathcal{T}}) =  \left\langle \Phi_r(\vec{x}),\Phi_r(\vec{y})\right\rangle
=  \mathclap{\;\; \sum_{\gamma \in {\binom{\mathcal{X}}{2}}}} \; \; \; \; \Phi_r(\vec{x})[\gamma] \Phi_r(\vec{y})[\gamma]
\label{Eq:kernel}
\end{equation}

Intuitively, the kernel value between $\vec{x}$ and $\vec{y}$, $K_1(\vec{x},\vec{y} \vert {\mathcal T})$ counts the number of pairs having same relative ordering with respect to both $\vec{x}$ and $\vec{y}$ minus those having different relative ordering. $K_2(\vec{x},\vec{y} \vert {\mathcal T})$, on the other hand, measures the similarity of $\vec{x}$ and $\vec{y}$ based on whether they rank similarly with respect to their distances from other objects.  In this work, we focus on $K_1$ and denote it by $K$. Note that all our results readily extend to $K_2$. We use $K$ to show the improvement achieved by triplets augmentation in the quality of data analytics tasks.

Given the triplets data $\cT$, our goal is to perform data analytics tasks (e.g., nearest neighbors, clustering, and classification) efficiently without using the kernel. Given that our data do not necessarily reside in a numeric feature space and the distance measure is not explicitly provided, we define centrality, median, and closeness as follows:

\begin{definition} \label{def:centrality}
	The \textbf{centrality} of an object in a dataset is how {\em close} or {\em similar} it is to all other objects.
	Centrality of $\vec{x}\in \mathcal{X}$ is defined as: $cent(\vec{x}) := \sum\limits_{\vec{y}\in {\mathcal X}} sim(\vec{x},\vec{y}) =: \sum\limits_{\vec{y}\in {\mathcal X}} -d(\vec{x},\vec{y}) $. 
\end{definition}
\begin{definition} \label{def:median} A \textbf{median} (or centroid) of the dataset is an object with the largest centrality. The median $\vec{x}_{med}$ of $\mathcal{X}$ is given by: $\vec{x}_{med} := \argmax \limits_{\vec{x}\in \mathcal{X}} \;\;cent(\vec{x})$.
\end{definition}
\begin{definition}\label{def:closeness}
	For two objects $\vec{x},\vec{y} \in \mathcal{X}$, {\bf closeness} of $\vec{y}$ to $\vec{x}$ is the rank of $\vec{y}$ by similarity to $\vec{x}$, more formally: $close_{\vec{x}}(\vec{y}) := (n - 1) -  \lvert \{\vec{z}\in \mathcal{X}, \vec{z}\neq \vec{x} : sim(\vec{x},\vec{z}) < sim(\vec{x},\vec{y}) \}\rvert$, i.e., $close_{\vec{x}}(\vec{y})$ is the index of $\vec{y}$ in the list of all objects in the decreasing order of similarity to $\vec{x}$.
\end{definition}
\begin{definition}\label{def:knn} The (ordered) set of {\bf $k$ nearest neighbors} of an object $\vec{x}\in \mathcal{X}$, $k\textsc{nn}(\vec{x})$, is denoted by $k\textsc{nn}(\vec{x}) := \{\vec{y} \;|\; close_{\vec{x}}(\vec{y}) \leq k\}$.
\end{definition}

In Section~\ref{sec_triplets_analysis}, we provide triplets-based algorithms to perform above defined analytics tasks.
To evaluate performance of our proposed methods, we introduce some noise in the triplets data to replicate the real-world scenarios. 

\subsection{Relative Noise Model}
We use \emph{relative noise model}~\cite{chumbalov2019learning} to flip a portion of input triplets to replicate the real settings, as in practical crowd-sourcing scenarios, where conflicts are generally unavoidable. In the relative noise model, a triplet can be flipped or wrong with probability proportional to the difference in distances that the triplet encodes. 

Given an anchor object $\vec{x}$ and two query points $\vec{y}$ and $\vec{z}$, if $\vec{y}$ and $\vec{z}$ are equidistant from $\vec{x}$, declaring any of them to be closer to $\vec{x}$ corresponds to a fair coin toss with equal chances of $(\vec{x},\vec{y},\vec{z})_A$ or $(\vec{x},\vec{z},\vec{y})_A$. As the difference in the distances of $\vec{y}$ and $\vec{z}$ from $\vec{x}$ increases, it is less likely to make such an error in perception. Geometrically, suppose $\cX \subset \mathbb{R}^t$, i.e., all objects are $t$-dimensional real vectors. Let $h_{\vec{y}\vec{z}}$ be the bisecting normal hyperplane to the line segment joining $\vec{y}$ and $\vec{z}$. Also, suppose that the coefficients of $h_{\vec{y}\vec{z}}$ are $a_{\vec{y}\vec{z}_1}, a_{\vec{y}\vec{z}_t},$ and $b_{\vec{y}\vec{z}}$. Any point $\vec{x}$ on $h_{\vec{y}\vec{z}}$ is equidistant from the query points, and the distance comparisons can give $(\vec{x},\vec{y},\vec{z})_A$ or $(\vec{x},\vec{z},\vec{y})_A$ with equal probability. If the anchor object $\vec{x}$ is in the lower halfspace (i.e., $\sum_{i=1}^t a_{\vec{y}\vec{z}_i}\vec{x}_i < b_{\vec{y}\vec{z}}$), then $\vec{y}$ is closer to $\vec{x}$ than $\vec{z}$ (i.e., $(\vec{x},\vec{y},\vec{z})_A$ is the correct triplet) and vice-versa (see Figure~\ref{fig:bisectinghyperplane}). In the probabilistic model, random white noise is introduced in distance comparison. Formally, the point $\vec{y}$ is declared to be closer to $\vec{x}$ than $\vec{z}$ with probability equal to $Pr[\sum_{i=1}^t a_{\vec{y}\vec{z}_i}\vec{x}_i + \epsilon < b_{\vec{y}\vec{z}}]$ and vice-versa, where $\epsilon \sim {\cal N}(0,\sigma_{\epsilon}^2)$ is the random noise. However, above model requires objects to be $\mathbb{R}^t$ and the coordinates are known. In our experiments, we set the probability of flipping a correct triplet $(\vec{x},\vec{y},\vec{z})_A$ to be equal to $1-|d'(\vec{x},\vec{y}) - d'(\vec{x},\vec{z})|$, where $d'(\cdot,\cdot)$ are the pairwise distances scaled to $[0,1]$. 

\section{Embedding Generation}\label{section:algorithm}
This section describes the representation of feature vector $\Phi_{1}$ that enables applications of set operations for kernel evaluation, efficient data augmentation, and computation of closeness from an object $\vec{x}$ and its approximate nearest neighbors.

For evaluating the centrality of objects and finding the median of the dataset, we give an abstract representation of similarity triplets in $\mathcal{T}$ (i.e., statements of type~\ref{tripAnchor}). This facilitates performing these analytics in linear time and space (i.e., linear in $|\mathcal{T}|$). Suppose that triplets in $\mathcal{T}$ are lexicographically sorted and let $|\mathcal{T}|=\tau$. This does not result in any loss of generality as $\mathcal{T}$ can be sorted as a preprocessing step in $O(\tau log \tau)$ time.

\subsection{Feature Vector Representation:} For a triplet $(\vec{x},\vec{y},\vec{z})_A$, we refer to $\vec{x}$ as the anchor of the triplet. For each $\vec{x}_i \in \mathcal{X}$ and $\vec{x}_i$ as the anchor, information in the triplets in $\mathcal{T}$ is encoded in a directed graph $G_i$ on the vertex set $\mathcal{X}$. The set of directed edges $E(G_i)$ in $G_i$ is defined as:
$$E(G_i) := \{(\vec{y},\vec{z})| \vec{y},\vec{z} \in \mathcal{X},\; (\vec{x}_i,\vec{y},\vec{z})_A \in \mathcal{T}\}$$

Note that edges directed from lower indexed to higher indexed objects in $G_i$ correspond to coordinates of $\Phi_1(\vec{x}_i)$ with values $1$ and $-1$ otherwise. The feature vector $\Phi_1(\vec{x}_i)$ given in \eqref{Eq:featuremap} is represented by $G_i$ as follows:
$$\Phi_{1}(\vec{x}_i)[\gamma] = \begin{cases} 1 & \text{ iff } \gamma = (\vec{x}_j,\vec{x}_k) \text{ and } (\vec{x}_j,\vec{x}_k) \in E(G_i)\\
-1& \text{ iff } \gamma = (\vec{x}_j,\vec{x}_k) \text{ and }  (\vec{x}_k,\vec{x}_j) \in E(G_i)\\
0 & \text{ otherwise }
\end{cases}$$

\begin{lemma}
	{For $1\leq i \leq n$, $G_i$ is a directed acyclic graph (DAG).}
\end{lemma}
\begin{proof}
	{Suppose, $G_i$ contains $t$ edges with no cycle and adding the new edge $(\vec{x},\vec{y})$ creates a cycle in $G_i$. A cycle can only be created by $(\vec{x},\vec{y})$ in $G_i$ only if there is already a directed path from $\vec{y}$ to $\vec{x}$. Let the directed path from $\vec{y}$ to $\vec{x}$ be the form of $\vec{y},\vec{z}_1,\vec{z}_2,...,\vec{z}_k,\vec{x}$. The path implies that $d(\vec{x}_i,\vec{y})<d(\vec{x}_i,\vec{z}_1)<d(\vec{x}_i,\vec{z}_2),...,d(\vec{x}_i,\vec{z}_k)<d(\vec{x}_i,\vec{x})$. By transitivity, we have $d(\vec{x}_i,\vec{y})<d(\vec{x}_i,\vec{x})$. However, the edge $(\vec{x},\vec{y})$ contradicts the inequality $d(\vec{x}_i,\vec{y})<d(\vec{x}_i,\vec{x})$. As in our setting, each distance query for a pair of objects gives exactly the same answer each time, the edge $(\vec{x},\vec{y})$ can not exist. This confirms the statement that directed graph $G_i$ made from triplets set with an anchor $\vec{x}_i$ is a DAG.}
\end{proof}

\subsection{Data Augmentation} Any reasonable notion of distance (or similarity) must admit the following property. {\em If an object $\vec{a}$ is closer to $\vec{x}$ than object $\vec{b}$, and object $\vec{b}$ is closer to $\vec{x}$ than object $\vec{c}$, then object $\vec{a}$ is closer to $\vec{x}$ than $\vec{c}$}, i.e. $$d(\vec{x},\vec{a})<d(\vec{x},\vec{b}) \wedge d(\vec{x},\vec{b}) < d(\vec{x},\vec{c}) \implies d(\vec{x},\vec{a}) < d(\vec{x},\vec{c})$$ In other words, $1\leq i \leq n$, edges of $G_i$ must induce a transitive relation on $\mathcal{X}$. We compute transitive closures of all $G_i$'s to obtain more sound triplets. In a digraph, $G = (V,E)$, for  $v \in V$, let $R(v)$ be the set of vertices reachable from $v$, i.e., vertices that have a path from $v$ in $G$. $R(v) = N^{+}(v) \bigcup_{u\in N^+(v)} R(u)$, where $N^+(v)$ is the set of out-neighbors of $v$. {Algorithm~\ref{findReachabilityAlgo} computes the reachability set of all vertices in a directed graph.}

\begin{algorithm}[t!]
	\caption{Augment ($DAG=(V,E)$), $|V| = n$}
	\label{findReachabilityAlgo}
	\begin{algorithmic}[1]
		\State $\mathcal{R} \gets \textsc{EmptySets}[1,\ldots,n]$ \vskip.03in
		\State $G_T \gets $ \Call{TopologicalSort}{$DAG$} \vskip.03in
		\For{ each node $ v_i \in V(G_T)$} \vskip.03in
		\State $\mathcal{R}[v_i] \gets  \Call{Rec-Reachability}{v_i}$  \vskip.03in
		\EndFor\vskip.03in
		\State $\left[V(G^*), E(G^*)\right] \gets \left[V(G_T)\,,\,\emptyset\,\right]$ \vskip.03in
		\State $E(G^*) \gets \textsc{adj-list}[v_i] = {\cal R}[v_i]$\vskip.03in
		\State \Return $G^*$ \vskip.03in
		\Function{Reachability}{node $v$}\vskip.03in
		\If {$\mathcal{R}[v]=\emptyset$} \vskip.03in
		\If {$ \vert N^{+}(v) \vert=0$}\vskip.03in
		\State \Return
		\Else 
		\State $\mathcal{R}[v] \gets \bigcup\limits_{v_j \in N^{+}(v)} \left\{{v_j}\cup \Call{Reachability}{v_j}\right\}$
		\EndIf
		\EndIf
		\State \Return $\mathcal{R}[v]$
		\EndFunction
	\end{algorithmic}
\end{algorithm}

\begin{definition}
	The \textbf{transitive closure} of a digraph $G$ is the graph $G^*$ that contains the edge $(u,v)$ if there is a path from $u$ to $v$ (see Figure~\ref{fig:transitivity}). More formally, $G^*$ is a digraph such that $V(G^*) = V(G)$ and $E(G^*) = \{(u,v):v\in R(u) \}$.\end{definition} \vskip-.05in
	
\begin{figure}[t!]
  \centering
          \includegraphics[scale = 0.8] {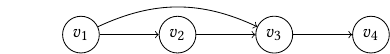}
\vskip.1in
        \includegraphics[scale = 0.8] {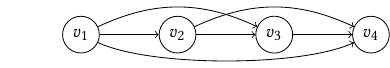}
        
    \caption{A graph $G$ (top) with its transitive closure $G^*$ (bottom).} 
	\label{fig:transitivity}
\end{figure}

\begin{lemma}
	The overall runtime of calling Algorithm~\ref{findReachabilityAlgo} on each of the $n$ $G_i$'s is $O(n^2 + |\cT^*|)$.
\end{lemma}
{\bf Proof:} Topological sorting of a DAG $G_i = (V_i,E_i)$ can be accomplished in $O(|V_i| + |E_i|)$~\cite{Cormen2009Algorithm}. The loop on line 3 iterates over each node, and the union operation in $\textsc{Reachability}({v_j})$ takes $O(deg^+(v_j))$ time. Thus, the runtime of one call to Algorithm~\ref{findReachabilityAlgo} takes $O(n+ \sum_j deg^+(v_j))$ time. Therefore, the total runtime of calling Algorithm~\ref{findReachabilityAlgo} on all $n$ graphs $G_i$, takes $O(n^2 + \sum_i |E(G_i)|) =  n^2 + \tau$. The time for constructing the augmented graphs in total takes time equal to the number of triplets in $\cT^*$. We get the total runtime of data augmentation as $O(n^2 + |\cT^*|)$. 
Note that the dimensionality of each feature vector is $O(n^2)$ and $\cT \subset {n\choose 3}$. In real settings, $|\cT|$ is likely to be significantly bigger than $O(n^2)$.

There are several benefits of this approach to data augmentation. First, having more data can lead to better knowledge extraction and machine learning. This is particularly true for data-intensive algorithms in deep learning. Second,  actively generating triplets that give us meaningful and novel information. Thus, in a crowdsourcing setting, where human annotators are actively queried, the number of queries (and their associated costs) can be minimized by only focusing on meaningful queries. Third, it can help in the identification of non-obvious conflicts in ordinal data provided in triplets, i.e., erroneous triplets, as described in an example below. 

For example consider a set of triplets $\cT = \{(\vec{x}_{i}, \vec{x}_{1}, \vec{x}_{2}),$ $(\vec{x}_{i}, \vec{x}_{2}, \vec{x}_{3}),$ $(\vec{x}_{i}, \vec{x}_{3}, \vec{x}_{5}),$ $(\vec{x}_{i}, \vec{x}_{5}, \vec{x}_{2}),$ $(\vec{x}_{i}, \vec{x}_{5}, \vec{x}_{4}) \}$. The corresponding $G_i$ is a conflict-free graph. However, the augmentation reveals all the indirect dependencies among nodes also. Augmenting $G_i$ yields triplet set as $\cT^* = \{(\vec{x}_{i}, \vec{x}_{1}, \vec{x}_{2}),$ $(\vec{x}_{i}, \vec{x}_{1}, \vec{x}_{3}),$ $(\vec{x}_{i}, \vec{x}_{2}, \vec{x}_{3}),$ $(\vec{x}_{i}, \vec{x}_{2}, \vec{x}_{5}),$ $(\vec{x}_{i}, \vec{x}_{3}, \vec{x}_{5}),$ $ (\vec{x}_{i}, \vec{x}_{5}, \vec{x}_{2}),$ $(\vec{x}_{i}, \vec{x}_{5}, \vec{x}_{3}), (\vec{x}_{i}, \vec{x}_{5}, \vec{x}_{4}) \}$. It is clear that we have two pairs of conflicting triplets $\{(\vec{x}_{i}, \vec{x}_{2}, \vec{x}_{5}),$ $(\vec{x}_{i}, \vec{x}_{5}, \vec{x}_{2})$ and $\{ (\vec{x}_{i}, \vec{x}_{3}, \vec{x}_{5}),$  $(\vec{x}_{i}, \vec{x}_{5}, \vec{x}_{3})\}$. The pictorial representation of $\cT$ and $\cT^*$ is shown in Figure~\ref{fig:hidden_conflict_detection_2}. Conflicting triplets can be dealt with in many different ways depending on the specific area of application. For instance, one can remove both or either of the triplets from the dataset based on user-defined criteria. 

\begin{figure}[t!]
	\centering
	\includegraphics[width=.7\linewidth,page=11]{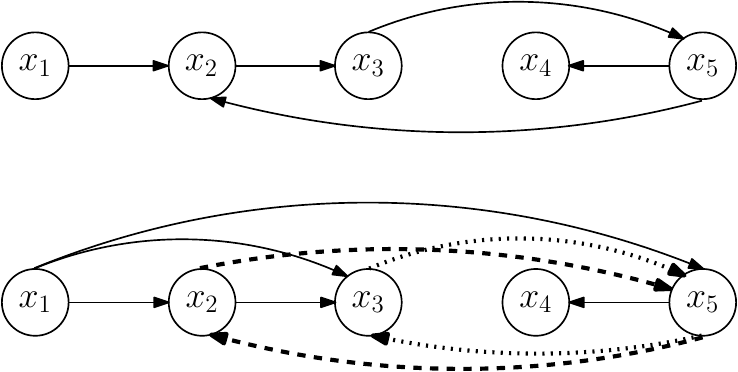}
	\caption{Hidden conflict detection using augmentation. A graph $G_i$ and the augmented graph $G_i^*$ (bottom). Pairs of dashed and dotted edges in $G_i^*$ are conflicting edges not obvious in $G_i$. }
	\label{fig:hidden_conflict_detection_2}
\end{figure}

In the case of noisy input data, directed graphs corresponding to anchor objects may have cycles, and algorithms to find transitive closure (augmentation) of directed graphs can be applied in such scenarios. The augmentation of directed graphs in such scenarios also leads to the propagation of errors in the graph. However, in the case of the relative noise model, the error remains much contained after augmentation as compared to the random noise model introduced in~\cite{Kleindessner2017lens,PerrotL19}. Results in section~\ref{section:experimental_evaluation} show that our methods for data analytics tasks are robust to error and can also be applied to directed graphs.

\begin{figure}[!t]
	\centering
	\includegraphics[width=.5\linewidth,page=9]{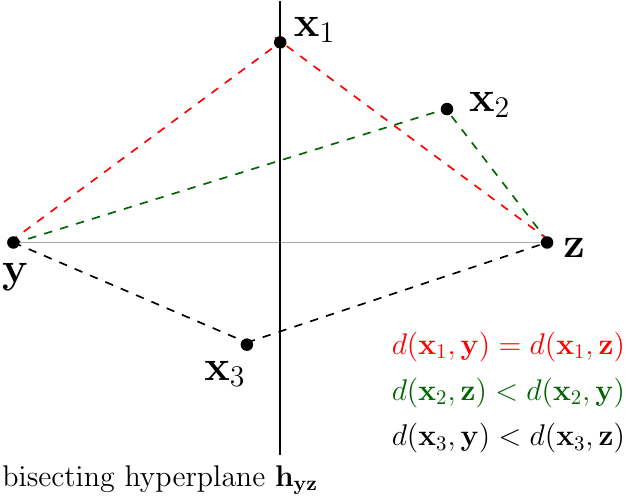}
	\caption{Distance comparison of two objects $\vec{y}$ and $\vec{z}$ from anchor points $\vec{x}_1,\vec{x}_2,$ and $\vec{x}_3$. The distance comparisons from anchor points that are close to the bisecting hyperplane $h_{\vec{y}\vec{z}}$ are more difficult to answer and more prone to error.}
	\label{fig:bisectinghyperplane}
\end{figure}

\section{Triplets based Data Analysis}\label{sec_triplets_analysis}
In this section, we propose algorithms for common data analysis tasks based on (i) kernel, (ii) directly on the input collection of triplets, and (iii) the augmented set of triplets.
First, we give an algorithm to compute the centrality of objects and median of the dataset from kernel to demonstrate that kernel computed from the augmented triplets significantly outperforms that based only on input triplets.  
Next, we described the \emph{kernel-free} data analysis based on both input and augmented triplets. 
Our algorithm estimates the nearest neighbors for each object directly from our representation of triplets, thus, circumventing the need for kernel computations. 
The approximate $k$-nearest neighbors are then used to populate the $k$-NN graph, on which we perform spectral clustering to cluster objects. For classification, we use the standard $k$-NN classifier.

\subsection{Centrality and Median Computation} 

We define an approximate centrality measure for object $\vec{x}$, $cent'(\vec{x})$ to quantify how close $\vec{x}$ is to other objects. Let $H$ be a ${n \times n}$ matrix with a row and column corresponding to each object in $\cX$. Let the similarity rank of $\vec{y}$ in row $\vec{x}$ of $K$ be $rank_{\vec{x}}(\vec{y}) = (n - 1) -  \lvert \{\vec{z}\in \mathcal{X}, : K(\vec{x},\vec{z}) < K(\vec{x},\vec{y}) \}\rvert $. Note that this is very similar to $close_{\vec{x}}(\vec{y})$, but it is based on $K$, instead of ${\cal S}$. The object closest to $\vec{x}$ has rank $1$, and the object farthest from $\vec{x}$ has rank $(n-1)$. The row in $H$ corresponding to $\vec{x}$ contains the similarity ranks of objects based on the $\vec{x}$th row of $K$. In other words, corresponding to each row in $K$, we have a permutation of integers in the range $[0,n-1]$ in $H$. Our approximate centrality score of object $\vec{x}$ is the $\ell_p$-norm of the $\vec{x}$th column of $H$. Formally, $$cent'(\vec{x}) \; := \; \|H(:,\vec{x})\|_p \;=\; \bigg(\sum_{y} H(\vec{y},\vec{x})^p\bigg)^{\nicefrac{1}{p}},$$ i.e., $cent'(\vec{x})$ aggregates the similarity rank of $\vec{x}$ with respect to all other objects. As smaller rank values imply more similarity of $\vec{x}$ in the respective row, the object with the smallest $cent'(\cdot)$ score is the most central object, which is our approximation for the median of the dataset. 

\begin{algorithm}[t!]
	\caption{: Compute Centrality($K$)}
	\label{computeCentralityAlgo}
	\begin{algorithmic}[1]
	\State $cent' \gets \textsc{zeros}[1,\ldots,n]$
		\For{$i=1$ to $n$}
		\State $H(i,:) \gets \Call{GetRanks}{K(i,:)}$ \Comment{returns ranks of objects in $i^{th}$ row of $K$}
		\EndFor
		
		\For{$i=1$ to $n$}
		\State $cent'[i] \gets \|{H(:,i)}\|_p$ \Comment{$\ell_p$-norm of $i^{th}$ column of $H$, $p\in [1,3]$} 
		\EndFor
		\Return $cent'[\cdots]$
		
	\end{algorithmic}
\end{algorithm}

\subsection{Nearest Neighbors} 
We approximate nearest neighbors of an object from triplets only avoiding computing the kernel. Using the information in $\cT$ stored in the DAGs, $G_i$ associated with each object $\vec{x}_i$, we can find upper and lower bounds on $close_{\vec{x}_i}(\vec{y})$. Note that $\vec{y}$ is closer to $\vec{x}_i$ than all elements $\vec{z} \in R(\vec{y})$ in $G_i$. This is so because $\vec{z}\in R(\vec{y})$ implies that $d(\vec{x}_i,\vec{y}) < d(\vec{x}_i,\vec{z}) \implies sim(\vec{x}_i,\vec{y}) > sim(\vec{x}_i,\vec{z})$. Let $deg_{G_i}^+(\vec{y})$ and $deg_{G_i}^-\vec{(}y)$ denote the out-degree and in-degree of $\vec{y}$ in $G_i$, respectively. Based on the above information, we have that 
\begin{equation}\label{eq:closeLower}
close_{\vec{x}_i}(\vec{y}) \geq deg_{G_i}^+(\vec{y}).
\end{equation}  We get an upper bound on closeness of $\vec{y}$ to  $\vec{x}_i$ as: 	\begin{equation}\label{eq:closeUpper}
close_{\vec{x}_i}(\vec{y}) \leq n - deg_{G_{i}}^-(\vec{y})
\end{equation}
Combining Equations \eqref{eq:closeLower} and \eqref{eq:closeUpper} we get $$deg_{G_i}^+(\vec{y}) \leq close_{\vec{x}_i}(\vec{y}) \leq n - deg_{G_{i}}^-(\vec{y}).$$

Our approximate closeness of an object $\vec{y}$ to $\vec{x}$, $close'_{\vec{x}}(\vec{y})$ is then an average of the upper and lower bounds in Equations \eqref{eq:closeLower} and \eqref{eq:closeUpper}. The approximate {$k$-nearest neighbors} of $\vec{x}\in \mathcal{X}$, are computed based on estimated closeness $close'_{\vec{x}}(\vec{y})$, i.e.
\begin{equation}\label{eq_approx_knn}
    k\textsc{nn}'(\vec{x}) := \{\vec{y} \;|\; close'_{\vec{x}}(\vec{y}) \leq k\}
\end{equation}

Note that nearest neighbors can be approximated from the kernel matrix, but since in many practical cases we are only interested in $k\textsc{nn}(\vec{x})$ for a fixed object, computing the whole kernel matrix is unnecessary. Since degree vectors can be maintained while populating the graphs $G_i$, runtime to compute $close'_{\vec{x}}(\vec{y})$ and $k\textsc{nn}'(\vec{x})$ is $O(1)$ and $O(n)$, respectively.

\subsection{Clustering and Classification} \label{sec:clustAndClassification} 
We construct the $k$-nearest neighborhood graph $k\textsc{nng}$~\cite{paredes2006practical} for $\mathcal{X}$ using  $\mathcal{T}$. The $k\textsc{nng}$ of a dataset $\mathcal{X}$ is a graph on vertex set $\mathcal{X}$ and object $\vec{x}$ is adjacent to $k$ vertices in $k\textsc{nn}(\vec{x})$. The $k_w\textsc{nng}$ is $k\textsc{nng}$ with edge-weights proportional to closeness of the adjacent vertices (i.e. weighted version of $k\textsc{nng}$ based on neighbors distance). For clustering $\mathcal{X}$, we apply spectral clustering~\cite{vonLuxburg2007} on $k\textsc{nng}$ for $\mathcal{X}$. For constructing $k\textsc{nng}$, we use the approximate $k$ nearest neighbors  $k\textsc{nn}'(\vec{x})$ (see Equation~\eqref{eq_approx_knn}) of each object $\vec{x}$. For the classification task, the well-known nearest neighbor classification method~\cite{peterson2009k} is used by taking a majority label among the labeled points in $k\textsc{nn}'(\vec{x})$.

\section{Experimental Setup}\label{section:experimental_evaluation}

Given that the computational complexity of our approach is better than the compared existing approaches, it is easier to scale our algorithms to Big data scenarios.
Therefore, we perform a set of experiments to evaluate the quality of results generated by our proposed algorithms in comparison with existing approaches. 
We evaluate the quality of our algorithms by appropriate comparison with analytics based on the true similarity matrix of $\cX$, i.e., $\cS(i,j) = sim(\vec{x}_i,\vec{x}_j)$.
Experiments are performed on an Intel Core i$7$ system with $16$GB memory. 

\subsection{Datasets Description} 
To perform experiments, we use four existing datasets from the well-known UC Irvine Machine Learning Repository~\footnote{http://archive.ics.uci.edu/ml/index.php}. The following datasets are used to randomly generate triplets.

\begin{itemize}
	\item \textsc{zoo} dataset consists of $101$ records and each record is represented by a $16$ dimensional feature vector describing the physical features of an animal. The dataset has $7$ different classes.
	
	\item \textsc{iris} is a flower dataset containing $150$ records, with $50$ records belonging to each of $3$ classes. A record has $4$ attributes showing lengths and widths of petals and sepals. 
	
	\item \textsc{glass} dataset contains $214$ objects and each object has $9$ features. Features show the number of components used in the composition of the glass. The dataset has $7$ classes.
	
	\item \textsc{moons} is a synthetically generated dataset of $500$ points that form two interleaving half circles. Each point in the dataset is $2$-dimensional, and the dataset has $2$ classes.
\end{itemize}

For synthetic datasets, we use feature vectors to generate similarity matrix $\mathcal{S}$ and distance matrix $\mathcal{D}$. We use distance metrics that are widely adopted in the literature for the respective datasets~\cite{PerrotL19}. We use Euclidean similarity metric for \textsc{iris},\textsc{glass} and \textsc{moons} datasets and Cosine similarity metric for \textsc{zoo} dataset. We use $\mathcal{D}$ and $\mathcal{S}$ only to generate triplets and to compare the effectiveness of our method. We randomly generate triplets by comparing the distances of two objects $y$ and $z$ from an anchor object $x$. A triplet $(x,y,z)$ is obtained by comparing $d(x,y)$ and $d(x,z)$ such that $d(x,y)<d(x,z)$. We generate $ \{1,5,10,20 \}~ \%$ of total possible triplets and we introduce \emph{relative error} $=\{0, 1,5,10,20 \}~ \%$ in generated triplets in our experiments.

\subsection{Compared Algorithms}

We show that data augmentation helps in improving the quality of the kernel matrix and the analytics performed on the kernel. We perform data analytics tasks like median computation, finding approximate nearest neighbors, classification, and clustering. We compare the median results of our approach with \textsc{crowd-median}~\cite{heikinheimo2013crowd} and \textsc{lensdepth}~\cite{Kleindessner2017lens}. Clustering results are compared with \textsc{lensdepth} and we compare classification results with \textsc{lensdepth} and \textsc{tripletboost}~\cite{PerrotL19}. Note that \textsc{crowd-median} works with~\ref{tripOutlier} form triplets only and \textsc{lensdepth} works with triplets of form~\ref{tripCentral}. In experiments, while comparing with \textsc{crowd-median} and \textsc{lensdepth}, we generate triplets of form~\ref{tripOutlier} and~\ref{tripCentral} respectively and then translated them to form~\ref{tripAnchor} triplets for our methods. Note that competitor methods \textsc{lensdepth} and \textsc{tripletboost} also work with erroneous triplets, however, the \emph{random error model} used by them is different from \emph{relative error model} used in this study. So, we make comparisons with these approaches in error-free settings only, and we show the impact of the error on our methods only. We have implemented \textsc{crowd-median}~\cite{heikinheimo2013crowd} and \textsc{lensdepth}~\cite{Kleindessner2017lens} algorithms while for \textsc{tripletboost}, we use the results reported in their paper~\cite{PerrotL19}.

\section{Results and Discussion}\label{sec_results_discussion}
In this section, we report experimental results averaged over $3$ runs. The results show a similar trend for the varying value of $\tau ~\%$ and error, so we report the representative results only.

\subsection{Kernel Matrix} 
The effectiveness of kernel matrix $K$ is to what extent $K$ \textit{agrees} with $\mathcal{S}$ and how well $K$ maintains the order of objects with respect to $\mathcal{S}$. We show that the augmented kernel $K^*$ computed from $\mathcal{T}^*$ is a closer approximation to $\mathcal{S}$ compared to the kernel $K$ computed from $\mathcal{T}$. Since only the ordering of distances is important, we report the row-wise rank correlation between $K$ and $\mathcal{S}$ and that between $K^*$ and $\mathcal{S}$. Figure~\ref{fig:row_wise_avg_rank_correlation} shows corresponding means and standard deviations (SD) of row-wise rank correlations with an increasing number of triplets, showing improvement in correlations, especially for the small number of triplets. SD are too small to be seen in the reported results. 

\begin{figure}[t!]
\centering
\begin{subfigure}{.25\textwidth}
  \centering
  \includegraphics[scale=0.735]{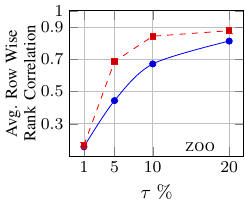}
\end{subfigure}%
\begin{subfigure}{.25\textwidth}
  \centering
  \includegraphics[scale=0.735]{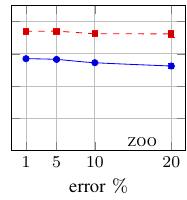}
\end{subfigure}%
\\
\begin{subfigure}{.25\textwidth}
  \centering
  \includegraphics[scale=0.735]{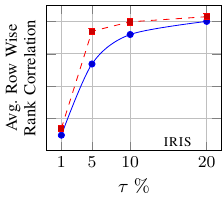}
\end{subfigure}%
\begin{subfigure}{.25\textwidth}
  \centering
  \includegraphics[scale=0.735]{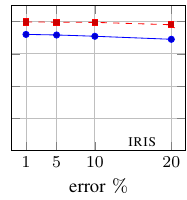}
\end{subfigure}%
\\
\begin{subfigure}{.25\textwidth}
  \centering
  \includegraphics[scale=0.735]{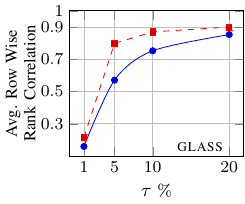}
\end{subfigure}%
\begin{subfigure}{.25\textwidth}
  \centering
  \includegraphics[scale=0.735]{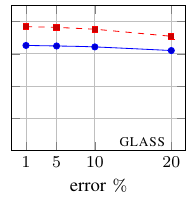}
\end{subfigure}%
\\
\begin{subfigure}{.25\textwidth}
  \centering
  \includegraphics[scale=0.735]{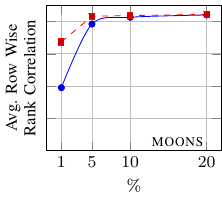}
\end{subfigure}%
\begin{subfigure}{.25\textwidth}
  \centering
  \includegraphics[scale=0.735]{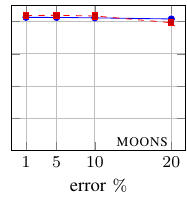}
\end{subfigure}%
\\
\begin{subfigure}{0.5\textwidth}
  \centering
  \includegraphics[scale=0.735]{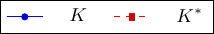}
\end{subfigure}%
\caption{Average row wise rank correlation of $K$ and $K^*$ with $\mathcal{S}$ (true similarity matrix) for different datasets. The left column shows the increase in correlation value with the increase in the number of error-free triplets. The right column shows the impact of varying error on triplets$(\tau =10 \%)$. A higher correlation value shows more agreement with $\mathcal{S}$.}
	\label{fig:row_wise_avg_rank_correlation}
\end{figure}

\begin{figure}[h!]
\centering
\begin{subfigure}{.25\textwidth}
  \centering
  \includegraphics[scale=0.735]{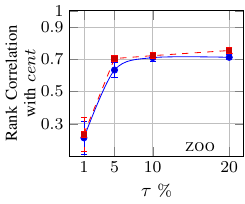}
\end{subfigure}%
\begin{subfigure}{.25\textwidth}
  \centering
  \includegraphics[scale=0.735]{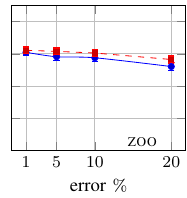}
\end{subfigure}%
\\
\begin{subfigure}{.25\textwidth}
  \centering
  \includegraphics[scale=0.735]{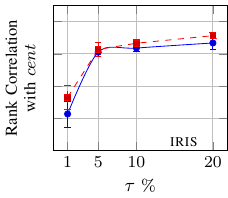}
\end{subfigure}%
\begin{subfigure}{.25\textwidth}
  \centering
  \includegraphics[scale=0.735]{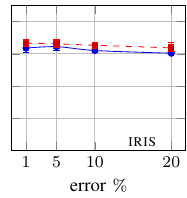}
\end{subfigure}%
\\
\begin{subfigure}{.25\textwidth}
  \centering
  \includegraphics[scale=0.735]{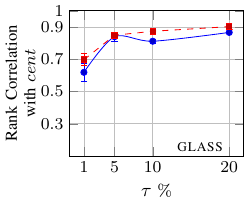}
\end{subfigure}%
\begin{subfigure}{.25\textwidth}
  \centering
  \includegraphics[scale=0.735]{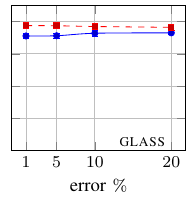}
\end{subfigure}%
\\
\begin{subfigure}{.25\textwidth}
  \centering
  \includegraphics[scale=0.735]{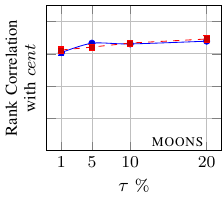}
\end{subfigure}%
\begin{subfigure}{.25\textwidth}
  \centering
  \includegraphics[scale=0.735]{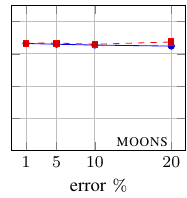}
\end{subfigure}%
\\
\begin{subfigure}{0.5\textwidth}
  \centering
  \includegraphics[scale=0.735]{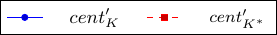}
\end{subfigure}%
\caption{Rank correlations of true and approximate centrality vectors. The $cent'_{K} $ and $cent'_{K^*}$ are centrality vectors computed from $K$ and $K^*$ respectively. The left column shows results for error-free settings with the varying number of triplets and the right column shows a slight decrease in correlation value for triplets $(\tau = 10\%)$ with the increase of error.}
	\label{fig:centrality_rank_correlation}
\end{figure}

\subsection{Centrality and Median} 
We demonstrate the quality of approximate centrality by showing the rank correlation between the true centrality vector $cent$ computed from $\mathcal{S}$ and the approximate centrality vectors ($cent'_K$ and $cent'_{K^*}$ computed from $K$ and $K^*$ respectively). In Figure~\ref{fig:centrality_rank_correlation}, the average rank correlation approaches $1$ with the increasing number of triplets and augmentation helps in improving the rank correlation in most of the cases. We use centrality vectors to compute the median of the dataset. 
\begin{figure}[h!]
\centering
\begin{subfigure}{.25\textwidth}
  \centering
  \includegraphics[scale=0.735]{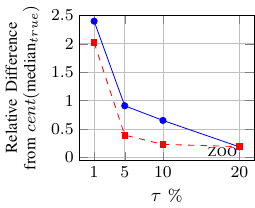}
\end{subfigure}%
\begin{subfigure}{.25\textwidth}
  \centering
  \includegraphics[scale=0.735]{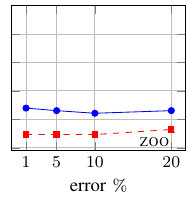}
\end{subfigure}%
\\
\begin{subfigure}{.25\textwidth}
  \centering
  \includegraphics[scale=0.735]{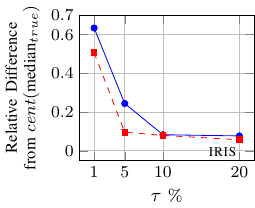}
\end{subfigure}%
\begin{subfigure}{.25\textwidth}
  \centering
  \includegraphics[scale=0.735]{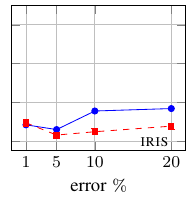}
\end{subfigure}%
\\
\begin{subfigure}{.25\textwidth}
  \centering
  \includegraphics[scale=0.735]{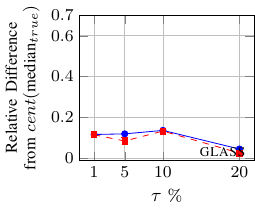}
\end{subfigure}%
\begin{subfigure}{.25\textwidth}
  \centering
  \includegraphics[scale=0.735]{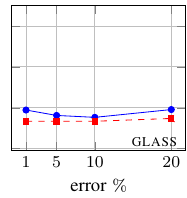}
\end{subfigure}%
\\
\begin{subfigure}{.25\textwidth}
  \centering
  \includegraphics[scale=0.735]{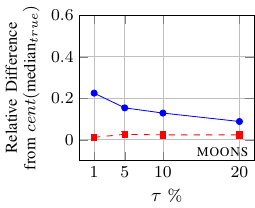}
\end{subfigure}%
\begin{subfigure}{.25\textwidth}
  \centering
  \includegraphics[scale=0.735]{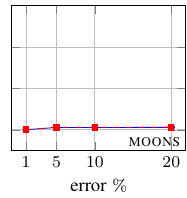}
\end{subfigure}%
\\
\begin{subfigure}{0.5\textwidth}
  \centering
  \includegraphics[scale=0.735]{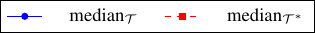}
\end{subfigure}%
\caption{Relative difference of median$_\cT$ and median$_{\cT^*}$ from the median$_{true}$. median$_{\cT^*}$ is generally closer to the median$_{true}$ compared to median$_\cT$. Plots on the left show results for triplets with $0\%$ error and plots on right show results for $\tau = 10 \%$ with varying error $\%$. Bottom left subplot is interpreted as for \textsc{moons} dataset, at $\tau\;\% = 10$, $cent($median$_\cT)$ and $cent($median$_{\cT^*})$ are $0.12$ and $0.02$ std-devs away from the median$_{true}$ respectively.}
	\label{fig:median_plot}
\end{figure}
Let median$_{true}$ be the median computed from the $cent$ and median$_{\cT}$ and median$_{\cT^*}$ are computed from the $cent'_K$ and $cent'_{K^*}$ respectively. 
To show that the approximate median lies close to median$_{true}$, we check how many standard deviations far the median$_{\cT}$ and median$_{\cT^*}$ are form the median$_{true}$. In Figure~\ref{fig:median_plot}, we report the relative difference of $cent$ value of the approximate median from the median$_{true}$ which is computed as $(cent(\text{median}_{true}) - cent(\text{median}_{\cT})) / \sigma(cent)$, where $\sigma(cent)$ is the standard deviation of $cent$ vector. 

We also compare median results with \textsc{crowd-median} algorithm. As \textsc{crowd-median} works on triplets of type~\ref{tripOutlier}, we generate~\ref{tripOutlier} type triplets and then transform them to~\ref{tripAnchor} triplets. We report the relative distance among median$_{true}$ and approximate medians which is computed as ${d(\text{median}_{true},\text{median}_{\cT})}/{\sigma(\mathcal{D})}$, where $\mathcal{D}$ is the true distance matrix. The comparison results with \textsc{crowd-median} are shown in Figure~\ref{fig:cmp_crowd_median}.

\begin{figure}[h!]
\centering
\begin{subfigure}{.25\textwidth}
  \centering
  \includegraphics[scale=0.735]{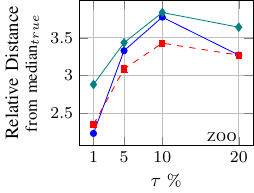}
\end{subfigure}%
\begin{subfigure}{.25\textwidth}
  \centering
  \includegraphics[scale=0.735]{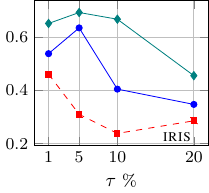}
\end{subfigure}%
\\
\begin{subfigure}{.25\textwidth}
  \centering
  \includegraphics[scale=0.735]{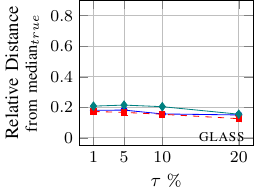}
\end{subfigure}%
\begin{subfigure}{.25\textwidth}
  \centering
  \includegraphics[scale=0.735]{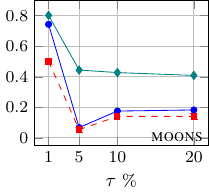}
\end{subfigure}%
\\
\begin{subfigure}{0.5\textwidth}
  \centering
  \includegraphics[scale=0.700]{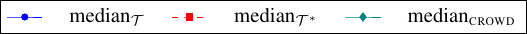}
\end{subfigure}%
\caption{Relative distance (std-devs) from median$_{true}$ of  \textsc{crowd-median} and our medians. For \textsc{crowd-median}, we generate triplets of type~\ref{tripOutlier} and translated them to type~\ref{tripAnchor} triplets. Results show that medians computed by our method are closer to the median$_{true}$ compared to median$_\textsc{crowd}$. $\tau \%$ shows the percentage of triplets of type~\ref{tripOutlier}.}
	\label{fig:cmp_crowd_median}
\end{figure}

\subsection{Nearest Neighbors} 

We show that our closeness-based method to approximately find the nearest neighbors performs well in practice. For each $x \in \mathcal{X}$, we compute true and approximate nearest neighbors denoted by $k\textsc{nn}(x)$ and $k\textsc{nn}'(x)$ respectively. To evaluate the effectiveness of our closeness based $k\textsc{nn}'(x)$, we apply standard $k$-means algorithm for clustering of all the data points in the $x^{th}$ row of similarity matrix $\mathcal{S}$. We find the \textit{closest} cluster from $x$ i.e. the cluster having maximum similarity with $x$. The similarity of cluster $C_i$ to $x$ is defined as $\dfrac{1}{|C_i|} \sum_{j\in C_i} sim(x,j)$. The performance of the proposed closeness approach is then measured by calculating the average intersection size of $k\textsc{nn}'(x)$ and the closest cluster. We make $ \lceil{\frac{n}{10}}\rceil$ clusters, where $n$ is the number of objects in the dataset. Here, the value of $\frac{n}{10}$ clusters is chosen empirically.  We report results for $k \in \{1,2\}$ in Figure~\ref{fig:knn} which shows that we achieve $60-80\%$ accuracy in finding the nearest neighbor, for the datasets used in experiments. We also observe that the closest cluster normally contains very few points, so the intersection percentage degrades with increasing $k$. Note that the closest cluster $C_i$ for each $x$ consists of true $|C_i|$ nearest neighbors of $x$. Thus, we do not report intersection results for $k\textsc{nn}(x)$ and the corresponding closest cluster.

\begin{figure}[h!]
\centering
\begin{subfigure}{.25\textwidth}
  \centering
  \includegraphics[scale=0.735]{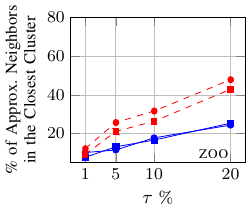}
\end{subfigure}%
\begin{subfigure}{.25\textwidth}
  \centering
  \includegraphics[scale=0.735]{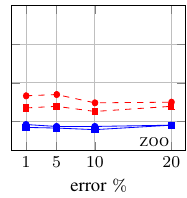}
\end{subfigure}%
\\
\begin{subfigure}{.25\textwidth}
  \centering
  \includegraphics[scale=0.735]{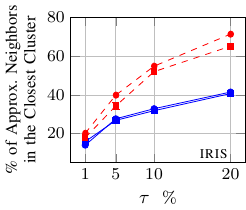}
\end{subfigure}%
\begin{subfigure}{.25\textwidth}
  \centering
  \includegraphics[scale=0.735]{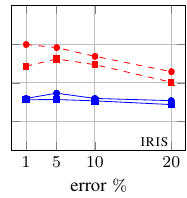}
\end{subfigure}%
\\
\begin{subfigure}{.25\textwidth}
  \centering
  \includegraphics[scale=0.735]{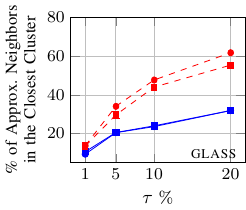}
\end{subfigure}%
\begin{subfigure}{.25\textwidth}
  \centering
  \includegraphics[scale=0.735]{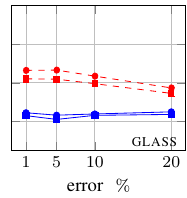}
\end{subfigure}%
\\
\begin{subfigure}{.25\textwidth}
  \centering
  \includegraphics[scale=0.735]{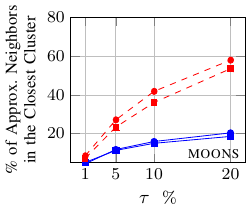}
\end{subfigure}%
\begin{subfigure}{.25\textwidth}
  \centering
  \includegraphics[scale=0.735]{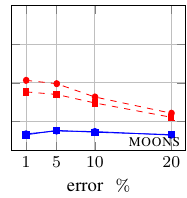}
\end{subfigure}%
\\
\begin{subfigure}{0.5\textwidth}
  \centering
  \includegraphics[scale=0.600]{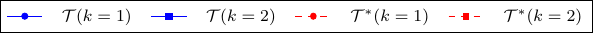}
\end{subfigure}%
\caption{Average percentage of approximate nearest neighbors $k\textsc{nn}'$ that belong to the closest cluster of each object. Plots in the left and right columns show the impact of varying $\tau$ (with no error) and varying  error (with $\tau = 10\%$). $\cT^*(k)$ represents results using augmented triplets for $k \in \{1,2\}$ neighbors.}
	\label{fig:knn}
\end{figure}



\subsection{Clustering} 
The goodness of the proposed approach is also evaluated by performing spectral clustering on the nearest neighborhood graph $k\textsc{nng}$. We construct $k\textsc{nng}$ and $k_w\textsc{nng}$ using approximate neighbors $k\textsc{nn}'$ as described in Section~\ref{sec:clustAndClassification}. We made a comparison with \textsc{lensdepth} to evaluate clustering quality. We implemented \textsc{lensdepth} algorithm~\cite{Kleindessner2017lens} using the same parameters as used in the original study (errorprob = $0$ and $\sigma = 5$). The nearest neighborhood graph is generated with $k=10$. In spectral clustering, we take the number of clusters equal to the number of classes in the dataset. Using augmented triplets $\cT^*$, performed slightly better than using $\cT$. Thus, for the sake of clarity, we only report results for augmented triplets. In Figure~\ref{fig:clusteringCmp}, we plot the purity of resulting clusters, which show improved results for $\cT^*$ as compared to \textsc{lensdepth}. 

\begin{figure}[h!]
\centering
\begin{subfigure}{.25\textwidth}
  \centering
  \includegraphics[scale=0.735]{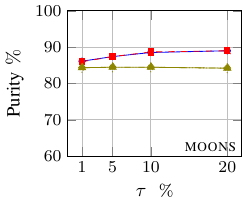}
\end{subfigure}%
\begin{subfigure}{.25\textwidth}
  \centering
  \includegraphics[scale=0.735]{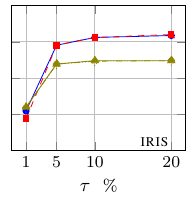}
\end{subfigure}%
\\
\begin{subfigure}{.25\textwidth}
  \centering
  \includegraphics[scale=0.735]{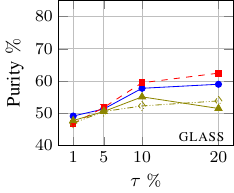}
\end{subfigure}%
\begin{subfigure}{.25\textwidth}
  \centering
  \includegraphics[scale=0.735]{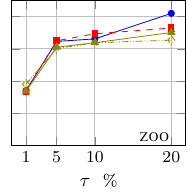}
\end{subfigure}%
\\
\begin{subfigure}{0.5\textwidth}
  \centering
  \includegraphics[scale=0.700]{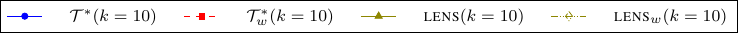}
\end{subfigure}%
\caption{Purity of clusterings using $k\textsc{nng}$, $k_w\textsc{nng}$, and \textsc{lensdepth} algorithm ($k=10$). We perform spectral clustering on $k\textsc{nng}$ and $k_w\textsc{nng}$ graphs and consider the same number of eigenvectors as of true classes in the datasets.}
	\label{fig:clusteringCmp}
\end{figure}



\subsection{Classification} We perform classification task using the $k\textsc{nn}$ classifier with train-test split of $70-30 \%$ for all datasets.  We make comparisons with \textsc{lensdepth} and \textsc{tripletboost} to evaluate the goodness of our approach. To make a comparison with \textsc{lensdepth}, we generate triplets of type~\ref{tripCentral} and then convert these triplets to type~\ref{tripAnchor} for compatibility with our proposed approach. Figure~\ref{fig:classificationCmp} shows comparison results with \textsc{lensdepth}. It is observed that the proposed method performs substantially better as the proportion of triplets increases. 

We make a comparison with \textsc{tripletboost} on two datasets: \textsc{moons} and \textsc{iris}.
Note that \textsc{tripletboost} also incorporates noisy triplets generated from the random noise model (which is different from the noise model used in this paper). So, we make a comparison with \textsc{tripletboost} on triplets with $0\%$ noise. On \textsc{iris} dataset, a comparison using $10\%$ triplets for three distance metrics is made to observe the impact of distance metric. These comparisons are shown in Figure~\ref{fig:classificationCmpWithBoosting}.

\begin{figure}[h!]
\centering
\begin{subfigure}{.25\textwidth}
  \centering
  \includegraphics[scale=0.735]{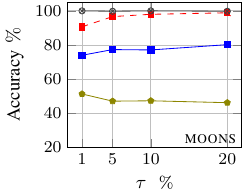}
\end{subfigure}%
\begin{subfigure}{.25\textwidth}
  \centering
  \includegraphics[scale=0.735]{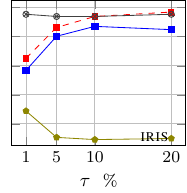}
\end{subfigure}%
\\
\begin{subfigure}{.25\textwidth}
  \centering
  \includegraphics[scale=0.735]{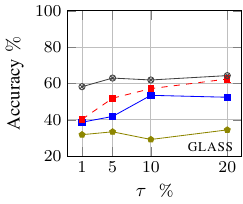}
\end{subfigure}%
\begin{subfigure}{.25\textwidth}
  \centering
  \includegraphics[scale=0.735]{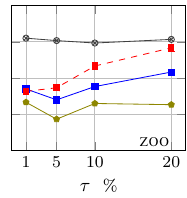}
\end{subfigure}%
\\
\begin{subfigure}{0.5\textwidth}
  \centering
  \includegraphics[scale=0.600]{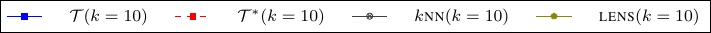}
\end{subfigure}%
\caption{Comparison of classification accuracy of $k\textsc{nn}$ classifier $(k=10)$ with \textsc{lensdepth} using $\cT$ and $\cT^*$ (augmented triplets). $k\textsc{nn}$ shows results based on true neighbors. In this case, {\small $\tau\;\%$} shows the percentage of triplets of type~\ref{tripCentral}.}
	\label{fig:classificationCmp}
\end{figure}

\begin{figure}[h!]
\centering
\begin{subfigure}{.25\textwidth}
  \centering
  \includegraphics[scale=0.70]{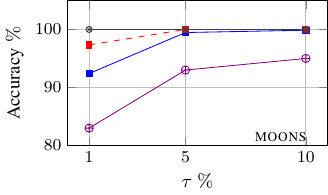}
\end{subfigure}%
\begin{subfigure}{.25\textwidth}
  \centering
  \includegraphics[scale=0.70]{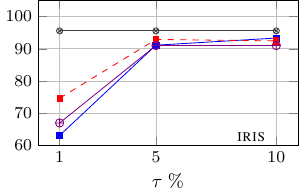}
\end{subfigure}%
\\
\begin{subfigure}{.25\textwidth}
  \centering
  \includegraphics[scale=0.70]{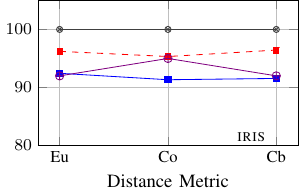}
\end{subfigure}%
\\
\begin{subfigure}{.5\textwidth}
  \centering
  \includegraphics[scale=0.600]{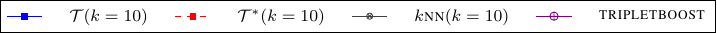}
\end{subfigure}%
\caption{Comparison of $k\textsc{nn}$ classification accuracy with \textsc{tripletboost} using $\cT$ and $\cT^*$. The bottom Figure plots the results on \textsc{iris} dataset with $\tau\;\% = 10$ generated from three different distance metrics, namely Euclidean (Eu), Cosine (Co), Cityblock (Cb).}
	\label{fig:classificationCmpWithBoosting}
\end{figure}

\section{Conclusion}\label{section:conclusion}
In this work, we propose a novel data augmentation technique for similarity triplets that expands the input data by inferring hidden information present in triplets. Data augmentation helps gather triplets for free and improves the quality of  kernel and data analytics tasks. We also present efficient algorithms for both supervised and unsupervised machine learning tasks without kernel evaluation. Empirical evaluation reveals that our techniques perform better than the competitor approaches and are also robust to error.
Future work includes the use of active learning to plan better data collection strategies and the identification of erroneous triplets in the input data.

\bibliographystyle{splncs04}
\bibliography{triplets_bib}
\end{document}